\documentclass{article}

\usepackage[preprint]{neurips_2023}
\usepackage{amsmath, amssymb, amsthm}
\usepackage{mathtools}
\usepackage{hyperref}
\usepackage{bm}
\newtheorem{theorem}{Theorem}

\usepackage{tikz}
\usetikzlibrary{positioning,arrows.meta,fit,shapes,calc}

\title{Transformers are Stateless Differentiable Neural Computers}

\author{%
  Bo Tang\thanks{Bo Tang is currently an Associate Professor at the Department of Electrical and Computer Engineering,
  Worcester Polytechnic Institute, Worcester, MA 01609. Email: btang1@wpi.edu} \\
  Department of Electrical and Computer Engineering\\
  Worcester Polytechnic Institute\\
  Worcester, MA 01609 \\
  \texttt{btang1@wpi.edu} \\
  \And
  Weiwei Xie \\
  Department of Data Science \\
  Worcester Polytechnic Institute \\
  Worcester, MA 01609 \\
  \texttt{wxie4@wpi.edu} \\
}

\begin{document}

\maketitle

\begin{abstract}
Differentiable Neural Computers (DNCs) were introduced as recurrent architectures equipped with an addressable external memory supporting differentiable read and write operations. Transformers, in contrast, are nominally feedforward architectures based on multi-head self-attention. In this work we give a formal derivation showing that a causal Transformer layer is \emph{exactly} a stateless Differentiable Neural Computer (sDNC) where (1) the controller has no recurrent internal state, (2) the external memory is a write-once matrix of value vectors, (3) content-based addressing via keys implements attention, and (4) multi-head attention corresponds to multiple parallel read heads. We further extend this equivalence to cross-attention, showing that encoder-decoder Transformers are precisely sDNCs with distinct read-from and write-to memories. Our results provide a unified memory-centric interpretation of Transformers and contribute to the ongoing effort to place modern large language models in a principled computational framework.
\end{abstract}

\section{Introduction}

Differentiable Neural Computers (DNCs) \cite{graves2014ntm} are neural architectures equipped with an external memory matrix and differentiable addressing mechanisms. They generalize Neural Turing Machines (NTMs) by adding content-based lookup, temporal link matrices, and gated write/erase operations. Transformers \cite{vaswani2017attention}, meanwhile, have become the dominant architecture in natural language processing, vision, and large-scale generative modelling \cite{elman_memory_transformer,schank2021attention}. They use multi-head attention to integrate information across tokens while retaining a stateless feedforward controller.

Although these architectures are usually treated as distinct, recent observations in theoretical and empirical studies suggest that Transformers behave like memory-access systems. In this paper we make this correspondence exact: we show that every (causal) Transformer layer implements a \emph{stateless} DNC (sDNC), a restricted version of a DNC in which the controller carries no recurrent state, the write mechanism is append-only, memory erasure or modification does not occur, content-based reads are identical to the attention mechanism.

We provide a full formalization, definitions, examples, and a theorem demonstrating equivalence. We also extend the analysis to \emph{cross-attention}, proving that encoder-decoder Transformers implement sDNCs with multiple external memories.

\section{Background}

Differentiable Neural Computers (DNCs), introduced by \citet{graves2014ntm}, extend the Neural Turing Machine architecture by augmenting a neural controller with a differentiable external memory. The controller is typically a recurrent neural network, such as an LSTM, possessing an internal hidden state $\mathbf{h}_t$ that evolves as a function of the input and past memory interactions. At each timestep the controller emits parameters that govern how the external memory $\mathbf{M}_t \in \mathbb{R}^{N \times W}$ should be accessed. Reading is performed by producing a distribution over memory locations, often determined through content-based addressing, which results in a set of read vectors formed through a weighted sum of memory rows. Writing involves an erase-and-add mechanism, where location-specific erasure vectors selectively remove information and addition vectors subsequently inject new content into the targeted memory slots. Related memory-augmented models include MANNs \cite{santoro2016meta} and Neural Turing Machines \cite{graves2014ntm}. The combination of recurrent controller dynamics and flexible read/write operations allows DNCs to learn algorithmic tasks involving long-range dependencies, structured data manipulation, and persistent memory utilization. 

Transformers, introduced by \cite{vaswani2017attention}, represent a markedly different architectural paradigm. Rather than relying on recurrence, Transformers employ an entirely feedforward computational structure in which all dependencies among tokens are mediated through attention mechanisms. Attention mechanisms used in Transformers are extensions of earlier alignment-based attention models in neural machine translation \cite{bahdanau2015neural, luong2015effective}. The central component is multi-head self-attention, wherein a sequence of input vectors is linearly projected into queries, keys, and values. Attention weights are computed through scaled dot products of queries and keys, normalized via a softmax function, and used to form a context-dependent weighted sum over values:
\[
\mathrm{Attn}(\mathbf{Q},\mathbf{K},\mathbf{V})
= \mathrm{softmax}\!\left(\frac{\mathbf{Q}\mathbf{K}^\top}{\sqrt{d_k}}\right)\mathbf{V}.
\]
By using multiple attention heads in parallel, the model can integrate information across different representation subspaces, enabling rich contextualization of each token with respect to all others. Following attention, a position-wise feed-forward network introduces further nonlinearity and capacity, while residual connections and layer normalization stabilize optimization and facilitate deep stacking of layers.

A crucial distinction between Transformers and traditional recurrent architectures, including the DNC controller, is the absence of any recurrent hidden state in the Transformer. Each layer processes the entire sequence in parallel, and all inter-token dependencies arise from attention over the inputs rather than from stepwise propagation of internal state. This means that the representational “memory” of a Transformer is distributed across the set of key and value vectors computed at each layer, forming an implicit, write-once storage mechanism that is subsequently queried by attention. The interpretation of attention as a memory access mechanism has been explored in several recent works \cite{elman_memory_transformer, katharopoulos2020transformers}. Despite lacking explicit recurrence or programmable write heads, Transformers exhibit behavior strongly reminiscent of memory-augmented systems, motivating the formal connection explored in this work.

\begin{figure}[t]
\centering

\begin{tikzpicture}[
  node distance=1.6cm and 2.0cm,
  every node/.style={font=\small},
  box/.style={draw, rounded corners, align=center, minimum width=3.1cm, minimum height=1.0cm},
  arrow/.style={-{Latex}, thick},
  dashedarrow/.style={-{Latex}, thick, dashed}
]

% -----------------------------------------------------
% Transformer (Left Column)
% -----------------------------------------------------

\node[box] (x) {$\mathbf{X}$ \\ Input sequence};
\node[box, below=of x] (qkv) {Linear projections \\ $\mathbf{Q},\mathbf{K},\mathbf{V}$};
\node[box, below=of qkv] (attn) {Self-attention \\ $\mathrm{softmax}\!\big(\frac{\mathbf{QK}^\top}{\sqrt{d_k}}\big)\mathbf{V}$};
\node[box, below=of attn] (z) {Output \\ $\mathbf{Z}$};

% Transformer label
\node[draw, rounded corners, fit=(x)(qkv)(attn)(z), inner sep=6pt, label={[yshift=0.2cm]above:\textbf{Transformer Layer}}] (transformer_box) {};

% Vertical arrows
\draw[arrow] (x) -- (qkv);
\draw[arrow] (qkv) -- (attn);
\draw[arrow] (attn) -- (z);

% -----------------------------------------------------
% sDNC (Right Column)
% -----------------------------------------------------

\node[box, right=3.5cm of x] (ctrl_in) {$\mathbf{x}_t$ \\ Controller input};
\node[box, below=of ctrl_in] (ctrl) {Feedforward controller \\ $g_\theta$};
\node[box, below=of ctrl] (kv) {Emit key/value \\ $\mathbf{k}_t,\mathbf{v}_t$};
\node[box, right=0.8cm of kv] (mem) {External memory \\ $\mathbf{M}=[\mathbf{v}_1;\dots;\mathbf{v}_T]$};
\node[box, below=of kv] (read) {Content read \\ $\mathbf{w}_t=\mathrm{softmax}(\mathbf{M}\mathbf{k}_t)$ \\ $\mathbf{r}_t=\mathbf{M}^\top\mathbf{w}_t$};

% Horizontal arrow from memory to read
\draw[arrow] (mem.south) -- (read.north);
\draw[arrow] (kv.south) -- (read.north);

% Vertical arrows
\draw[arrow] (ctrl_in) -- (ctrl);
\draw[arrow] (ctrl) -- (kv);
\draw[arrow] (kv) -- (mem);

% Surround sDNC box
\node[draw, rounded corners, fit=(ctrl_in)(ctrl)(kv)(mem)(read), inner sep=9pt, label={[yshift=0.2cm]above:\textbf{Stateless DNC (sDNC)}}] (sdnc_box) {};

% -----------------------------------------------------
% Correspondence arrows (Transformer → sDNC)
% -----------------------------------------------------

\draw[dashedarrow] (qkv.east) -- node[above,sloped]{keys/values $\Leftrightarrow$ $\mathbf{k}_t,\mathbf{v}_t$} (kv.west);

\draw[dashedarrow] (attn.east) -- node[above,sloped]{
\begin{tabular}{c}
attention \\
$\Leftrightarrow$\\
content-based read
\end{tabular}
} (read.north west);

\draw[dashedarrow] (z.east) -- node[above,sloped]{output $\Leftrightarrow$ read vector} ($(read.south west)+(0.0,0.05)$);

\draw[dashedarrow] (x.east) -- node[above,sloped]{
\begin{tabular}{c}
input tokens \\
$\Leftrightarrow$\\
controller input
\end{tabular}
} (ctrl_in.west);

\end{tikzpicture}

\caption{
Diagram showing the structural equivalence between a Transformer layer and a stateless Differentiable Neural Computer (sDNC). A Transformer’s linear projections produce keys and values analogous to the sDNC controller’s emissions. The self-attention mechanism corresponds exactly to content-based reads from an external memory composed of value vectors. Transformer outputs match sDNC readouts for each position.
}
\label{fig:transformer_sDNC_relation}
\end{figure}
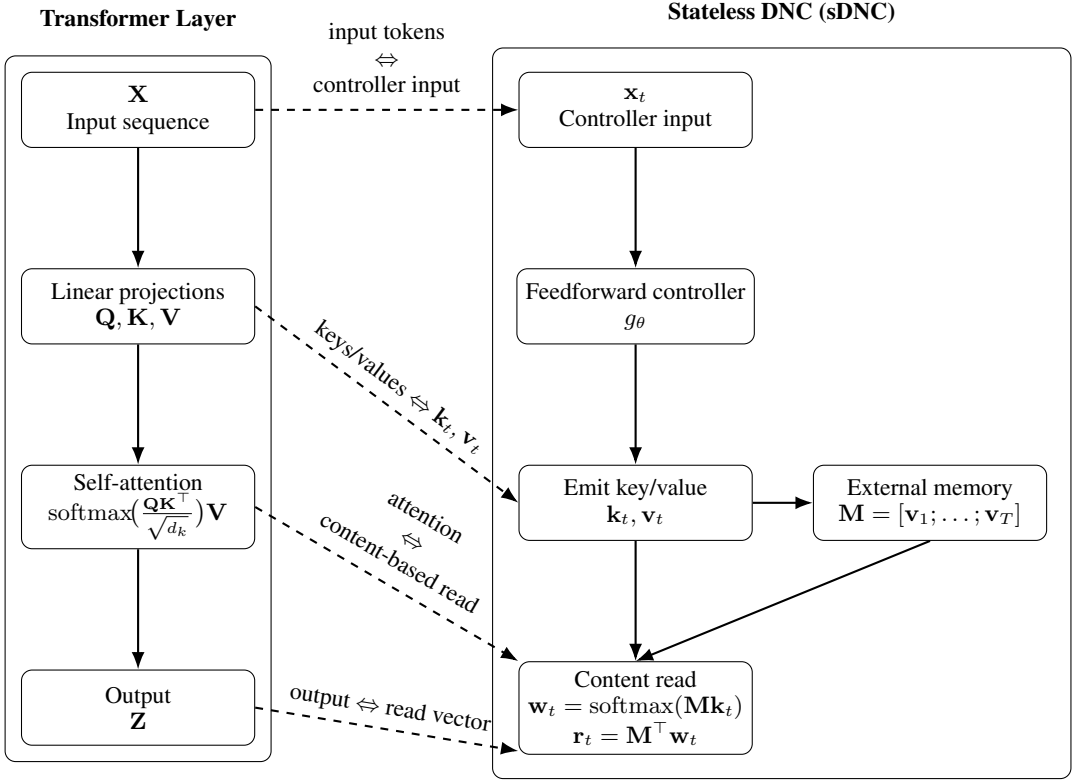

\section{Stateless Differentiable Neural Computers}

The DNC couples a recurrent controller to a differentiable external memory that supports programmable read and write operations. While this architecture provides substantial representational flexibility, it also introduces complexities in training, control flow, and memory management. In particular, a DNC must learn both \emph{how} to update its internal hidden state and \emph{when} to erase, modify, or write to its memory. These complicate optimization and often lead to fragile behavior in long sequences. Transformers, by contrast, forego recurrence and explicit write control entirely, yet display memory-like behavior and support long-context reasoning at scale.

To bridge these perspectives, we introduce a simplified variant of the DNC, which we refer to as the \emph{stateless DNC} (sDNC). The sDNC retains the essential computational structure of a memory-augmented neural machine. It is a controller that produces keys and values, a growing external memory that stores these values, and read heads that retrieve information through content-based lookup, but it removes the recurrent controller state and all forms of learned memory modification. This restricted formulation reveals that the core operations of the Transformer emerge directly from those of a DNC, once the architecture is constrained to a feedforward controller and a write-once memory. The sDNC thus serves as the minimal memory-augmented model capable of reproducing the behavior of a Transformer layer. Formally, we define the architecture, shown in Figure \ref{fig:transformer_sDNC_relation}, as follows.

\textbf{Stateless DNC (sDNC):} A stateless Differentiable Neural Computer consists of the following components. First, the system includes a feedforward controller $g_\theta$ that maps each input vector $\mathbf{x}_t$ to a key--value pair $(\mathbf{k}_t, \mathbf{v}_t)$. Unlike the recurrent controllers used in traditional memory-augmented networks, $g_\theta$ does not maintain an internal hidden state across timesteps. Each transformation is therefore memoryless with respect to the controller, and all temporal dependencies are mediated exclusively through interactions with the external memory:
\[
(\mathbf{k}_t,\mathbf{v}_t) = g_\theta(\mathbf{x}_t).
\]

Second, the architecture contains a write-once external memory that accumulates the value vectors generated by the controller. After processing $t$ inputs, the memory takes the form
\[
\mathbf{M}_t = 
  \begin{bmatrix}
    \mathbf{v}_1^\top \\
    \vdots \\
    \mathbf{v}_t^\top
  \end{bmatrix}
  \in \mathbb{R}^{t \times W},
\]
where each row corresponds to a stored representation produced at an earlier timestep. No erasure, overwriting, or gating mechanism is provided; the memory grows monotonically as the sequence is processed. This write-once design mirrors the behavior of value matrices in standard Transformers \cite{press2021memorizing, bahdanau2015neural}. This simplifies the architecture considerably while aligning it with the Transformer, whose internal representations (e.g., the value vectors at each layer) are likewise computed once and then left unchanged for subsequent attention-based queries.

Third, the model is equipped with $H$ content-based read heads. Each head receives a key vector $\mathbf{k}_t^{(h)}$ and produces a probability distribution over memory locations by comparing the key with every row of the memory matrix. These read weights take the form 
\[
\mathbf{w}_t^{(h)} = 
\mathrm{softmax}\!\left(
  \frac{1}{\sqrt{W}} \mathbf{M}_t \mathbf{k}_t^{(h)}
\right),
\]
corresponding to a soft address over all previously stored values. The weighted sum of memory rows yields the readout vector 
\[
\mathbf{r}_t^{(h)} = \mathbf{M}_t^\top \mathbf{w}_t^{(h)}.
\]
This mechanism directly parallels multi-head attention in Transformers, where each head computes dot-product similarities between queries and keys, normalizes the scores using a softmax, and aggregates values accordingly.

Importantly, the sDNC omits the traditional DNC’s recurrent hidden state, erase-and-add write operations, temporal linkage structure, and location-based addressing. All behavioral complexity arises solely from the feedforward generation of key--value pairs and the differentiable content-based lookup in the external memory. As we demonstrate in subsequent sections, once a DNC is constrained in this way, the remaining model is not merely reminiscent of a Transformer layer: it is mathematically equivalent to one.

This formalization highlights a crucial conceptual point. The Transformer’s
attention mechanism does not merely resemble a DNC read head, but it \emph{is} one,
implemented with fixed equations and without recurrent controller dynamics or
learned write policies. Similarly, the key and value projections in a Transformer
serve exactly the role played by the controller's emitted key--value vectors in a
memory-augmented network. The widely-used interpretation of Transformers as
attentional or relational architectures therefore admits a more computationally
grounded description: they are memory-based architectures in which all memory
interactions are restricted to content-based reads of a write-once memory.
This viewpoint not only provides a principled understanding of the capabilities
and limitations of Transformers but also establishes a foundation for unifying
them formally with the broader class of differentiable memory systems.

\section{Main Result: Transformers are sDNCs}

We now state the equivalence theorem. Similar equivalence observations have been noted in restricted settings in prior work \cite{katharopoulos2020transformers}. 

\begin{theorem} 
Consider a single causal Transformer self-attention layer with parameters $(W_Q,W_K,W_V,W_O)$ acting on a sequence $\mathbf{X} = (\mathbf{x}_1,\dots,\mathbf{x}_T)$. Let $\mathbf{z}_t$ denote its attention output at position $t$. Then there exists an sDNC such that its read vector $\mathbf{r}_t$ at time $t$ equals $\mathbf{z}_t$ for all $t$. Moreover, multi-head attention corresponds to $H$ independent read heads, and the Transformer output projection $W_O$ corresponds to mixing the read vectors.
\end{theorem}

\begin{proof}
For each input $\mathbf{x}_t$, define
\[
  \mathbf{k}_t = W_K^\top \mathbf{x}_t,
  \qquad
  \mathbf{v}_t = W_V^\top \mathbf{x}_t.
\]
Define memory $\mathbf{M}_t$ as the vertical concatenation of past value vectors. For causal attention, the attention distribution at time $t$ is
\[
  \alpha_{t,j} = \frac{\exp\!\left(
    \frac{\mathbf{k}_t^\top \mathbf{v}_j}{\sqrt{d_k}}
  \right)}{
    \sum_{\ell=1}^t \exp\!\left(
      \frac{\mathbf{k}_t^\top \mathbf{v}_\ell}{\sqrt{d_k}}
   \right)}.
\]
This equals
\[
  \mathbf{w}_t = 
    \mathrm{softmax}\!\left(\frac{1}{\sqrt{d_k}}\mathbf{M}_t\mathbf{k}_t\right),
\]
the content-based read weights of the sDNC.

The sDNC read vector is thus
\[
  \mathbf{r}_t
  = \mathbf{M}_t^\top \mathbf{w}_t
  = \sum_{j=1}^t \alpha_{t,j}\mathbf{v}_j,
\]
which matches the Transformer attention output $\mathbf{z}_t$ prior to output projection. Applying $W_O$ yields the complete attention output.

The extension to multi-head attention is immediate: each head corresponds to an independent key--value transformation and independent memory read, and $W_O$ linearly mixes the concatenated read vectors.
\end{proof}

\section{Cross-Attention as Multi-Memory sDNC}

Encoder-decoder architectures introduce \emph{cross-attention}, in which decoder queries attend to encoder keys/values. Cross-attention parallels classical sequence-to-sequence attention mechanisms \cite{bahdanau2015neural}. We show that this corresponds to an sDNC with two external memories: \textit{decoder memory} which is causal and write-once, produced from decoder states; and \textit{encoder memory} which is fixed and populated by encoder values.

\textbf{Two-Memory sDNC:} An sDNC with two memories has
\[
  \mathbf{M}^{(\text{dec})}_t = 
    \begin{bmatrix}
      \mathbf{v}^{(\text{dec})}_1 \\
      \vdots\\
      \mathbf{v}^{(\text{dec})}_t
    \end{bmatrix},
  \qquad
  \mathbf{M}^{(\text{enc})} =
    \begin{bmatrix}
      \mathbf{v}^{(\text{enc})}_1 \\
      \vdots\\
      \mathbf{v}^{(\text{enc})}_S
    \end{bmatrix},
\]
and provides two families of read heads: causal self-attention reads from $\mathbf{M}^{(\text{dec})}_t$, while cross-attention reads from $\mathbf{M}^{(\text{enc})}$.

Next we show that cross-attention equals sDNC multi-memory read. Let $\mathbf{M}^{(\text{enc})}$ be the matrix of encoder values. For each decoder timestep $t$, cross-attention computes
\[
  \mathbf{z}_t^{(\text{cross})}
  = \mathrm{softmax}\!\left(
      \frac{\mathbf{q}_t^{(\text{dec})} \left(\mathbf{M}^{(\text{enc})}W_K^{(\text{enc})}\right)^\top}{\sqrt{d_k}}
    \right)
    \mathbf{M}^{(\text{enc})}W_V^{(\text{enc})}.
\]
This is exactly the read operation of an sDNC with memory $\mathbf{M}^{(\text{enc})}$ and controller key $\mathbf{k}_t^{(\text{dec})}$. Cross-attention uses the encoder memory as the value matrix and its key matrix as the addressing basis. The softmax over dot products is identical to the sDNC content-based read; there is no dependency on decoder memory in this operation, and no recurrence. Therefore cross-attention is precisely an sDNC read from an independent write-once memory.

\section{Discussion}

This work formally establishes that Transformer layers implement stateless DNCs. The mapping is exact: the external memory in a Transformer layer corresponds to the matrix of value vectors; the controller is the feedforward computation that produces keys and values; and self-attention and cross-attention are nothing more than content-based read operations into one or more write-once memories. Residual connections and feed-forward (MLP) blocks serve as local post-processing mechanisms that transform each position’s representation independently, without introducing recurrent state.

Beyond formal equivalence, this perspective reveals deeper conceptual parallels that shed light on why Transformers scale and generalize so effectively. Traditional memory-augmented neural networks such as Neural Turing Machines and DNCs were designed explicitly to address tasks requiring algorithmic reasoning, variable binding, and persistent memory management. Transformers, despite lacking an explicit instruction set or recurrent control logic, display many of these behaviors in practice: they retrieve distant information through attention, store intermediate computations in value vectors, and coordinate information flow across layers through a combination of content-based access and hierarchical transformation. This connection helps explain why Transformers can exhibit algorithmic behavior \cite{wei2022chain}. Interpreting a Transformer as a stateless DNC clarifies these phenomena: the model implements a fixed memory interaction protocol that is simple enough to train at scale but expressive enough to support compositional and algorithmic behavior.

The perspective also highlights what Transformers \emph{lack} relative to full DNCs. A classical DNC can choose when and where to write, update, erase, or reuse memory locations; it operates with explicit temporal link matrices and can learn to traverse memory in structured ways independent of content. Transformers, by contrast, operate with strictly write-once memories tied to input positions and layers. They lack the ability to learn dynamic memory updates or long-term persistence across segments. As a result, modern Transformer systems require auxiliary mechanisms \cite{bahdanau2015neural} such as retrieval augmentation, key–value caching, or windowed attention to extend their effective memory horizon. From the sDNC perspective, this is equivalent to manually managing external memory that a DNC would naturally learn to manipulate.

These observations suggest that combining DNC-style capabilities with Transformer architectures may yield models with improved reasoning, efficiency, and robustness. Several promising directions emerge from this synthesis. One possibility is to incorporate selective write mechanisms, enabling a Transformer layer to update its value memory dynamically rather than write strictly once per token. This could mitigate the quadratic memory growth associated with long contexts and allow the model to preserve important information while discarding irrelevant content. Several recent works propose selective write or routing operations reminiscent of DNC functionality \cite{csordas2021neural}. Another avenue is to introduce differentiable erase or modify operations, enabling Transformers to maintain long-term state that persists across segments or even across tasks. Such extensions would bring the architecture closer to a full DNC, potentially yielding models with more stable memory utilization and stronger algorithmic capabilities.

Likewise, temporal-link matrices or differentiable pointers, central components of DNCs, could enrich Transformer attention with structured traversal abilities. Whereas self-attention selects memory positions based solely on content similarity, a DNC’s temporal memory introduces a notion of sequence order beyond positional encodings, allowing more explicit representation of chains of reasoning. Integrating such mechanisms into Transformer layers may help models reconstruct multi-step relationships more reliably and interpretably.

Finally, viewing Transformers as sDNCs provides a principled lens for analyzing recent architectural innovations, such as linear attention, state-space models, and recurrent memory Transformers. Many of these architectures can be reinterpreted as introducing learned write policies, constrained read mechanisms, or alternative memory organizations, all of which align naturally with the broader DNC framework. This unifying view invites a systematic exploration of hybrid models that blend the simplicity, scalability, and parallelizability of Transformers with the structural flexibility of DNC-inspired memory systems. Such models may ultimately offer the best of both worlds: the performance and stability of large-scale Transformers and the expressive, algorithm-like behavior of differentiable memory machines.

\section{Conclusion}

We have shown that Transformers can be interpreted as stateless Differentiable Neural Computers, where the value matrix serves as a write-once memory and attention implements content-based reads. This formulation provides a unified view of attention-based architectures within the broader framework of differentiable memory systems. By clarifying how information is stored, retrieved, and transformed across layers, the sDNC perspective offers a principled foundation for analyzing Transformer behavior and motivates new hybrid architectures that integrate explicit memory operations with the scalability and simplicity of Transformer models.

\section*{Acknowledgment}
This material is based upon work supported in part by NSF under Award IIS-2325863. Any opinions, findings, and conclusions or recommendations expressed in this publication are those of the author(s) and do not necessarily reflect the views of the NSF.

\bibliographystyle{plain}
\bibliography{reference}

\end{document}